\documentclass[sigconf]{acmart}
\AtBeginDocument{%
  }

\usepackage{amsfonts,amssymb,bbding,threeparttable}
\usepackage{xspace}

\newcommand{\bpi}{\boldsymbol{\pi}}

\newcommand{\sig}{\rlap{$^*$}}

\newtheorem{thm}{Theorem}

\newtheorem{lem}{Lemma}

\PassOptionsToPackage{square,numbers,sort&compress}{natbib}

\definecolor{abl}{HTML}{918cc2}

\usepackage{subfigure}
\usepackage{multirow}
\usepackage{multicol}
\usepackage{caption}
\usepackage{latexsym}
\usepackage{mflogo}
\usepackage{graphics}
\usepackage{enumitem}
\usepackage{xcolor}
\usepackage{algorithm}
\usepackage[noend]{algpseudocode}
\newsavebox\CBox

\newcommand{\E}{\mathbb{E}}
\newcommand{\p}{\rho}

\newcommand{\pt}{\rho^\mathrm{T=1}}
\newcommand{\pc}{\rho^\mathrm{T=0}}

\newcommand{\loss}{l_{\psi,\phi}}
\newcommand{\epehe}{\epsilon_{\mathrm{PEHE}}}

\graphicspath{ {./images/} }

\copyrightyear{2025}
\acmYear{2025}
\acmConference[KDD '25] {Proceedings of the 31st ACM SIGKDD Conference on Knowledge Discovery and Data Mining V.2}{August 3--7, 2025}{Toronto, ON, Canada.}
\acmBooktitle{Proceedings of the 31st ACM SIGKDD Conference on Knowledge Discovery and Data Mining V.2 (KDD '25), August 3--7, 2025, Toronto, ON, Canada}
\acmDOI{10.1145/3711896.3737092}
\acmISBN{979-8-4007-1454-2/2025/08}
\acmConference[KDD ’25]{}{August 3--7, 2025}{Toronto, ON, Canada}

\begin{document}

\title{Proximity Matters: Local Proximity Enhanced Balancing for Treatment Effect Estimation}
\author{Hao Wang}
\affiliation{
  \institution{College of Control Science and Engineering, Zhejiang University}
  \city{Hangzhou}
  \country{China}}

\author{Zhichao Chen}
\affiliation{
  \institution{College of Control Science and Engineering, Zhejiang University}
  \city{Hangzhou}
  \country{China}}

\author{Zhaoran Liu}
\affiliation{
  \institution{College of Control Science and Engineering, Zhejiang University}
  \city{Hangzhou}
  \country{China}}

\author{Xu Chen}
\affiliation{
  \institution{\mbox{Gaoling School of Artificial Intelligence} \mbox{Renmin University of China}}
  \city{Beijing}
  \country{China}}

\author{Haoxuan Li}
\authornote{Haoxuan Li and Zhouchen Lin are the corresponding authors.}
\affiliation{%
  \institution{Center for Data Science\\Peking University}
  \city{Beijing}
  \country{China}
}

\author{Zhouchen Lin}
\authornotemark[1]
\authornote{Zhouchen Lin is now also working at State Key Lab of
General AI \& Institute for AI, Peking University, and Pazhou Lab (Huangpu), Guangzhou, China.}
\affiliation{%
  \institution{School of Intelligence Science and Technology, Peking University}
  \city{Beijing}
  \country{China}
}

\renewcommand{\shortauthors}{Hao Wang et al.}

\begin{abstract} 
Heterogeneous treatment effect (HTE) estimation from observational data poses significant challenges due to treatment selection bias. Existing methods address this bias by minimizing distribution discrepancies between treatment groups in latent space, focusing on global alignment. However, the fruitful aspect of local proximity, where \textit{similar units exhibit similar outcomes}, is often overlooked.  In this study, we propose \textbf{Pro}ximity-enhanced  \textbf{C}ounterFactual \textbf{R}egression (CFR-Pro) to exploit proximity for enhancing representation balancing within the HTE estimation context. Specifically, we introduce a pair-wise proximity regularizer based on optimal transport to incorporate the local proximity in discrepancy calculation. However, the curse of dimensionality renders the proximity measure and discrepancy estimation ineffective—exacerbated by limited data availability for HTE estimation. To handle this problem, we further develop an informative subspace projector, which trades off minimal distance precision for improved sample complexity. Extensive experiments demonstrate that CFR-Pro accurately matches units across different treatment groups, effectively mitigates treatment selection bias, and significantly outperforms competitors. Code is available at \url{https://github.com/HowardZJU/CFR-Pro}.
\end{abstract}



\maketitle

\section{Introduction}\label{sec:introduction}

Estimating heterogeneous treatment effect (HTE) through randomized controlled trials is fundamental in causal inference, widely applied in various domains such as healthcare~\citep{drnet,escfr}, e-commerce~\citep{uplift,wu2022opportunity,escm}, and education~\citep{cordero2018causal}. While randomized controlled trials are considered as the golden standard in HTE estimation~\citep{pearl2018book}, their availability is often limited by significant financial and ethical constraints~\cite{li2023www,liremoving,xiao2024addressing,zhou2025two}. Consequently, there is increasing reliance on observational data for HTE estimation, driven by its broader availability and the feasibility of post-marketing surveillance as a cost-effective alternative to clinical trials~\cite{wuite,DBLP:conf/icml/LiZCGLW23}.

Estimating HTE from observational data is challenging primarily due to: (1) the absence of counterfactuals, where only one potential outcome is observable; and (2) treatment selection bias, where non-random treatment assignments cause covariate shifts between treated and untreated groups, thereby affecting the generalizability of outcome estimators~\cite{wustable,wang2023out,zheng2025adaptive}. Traditional meta-learners address the counterfactual problem by segmenting HTE estimation into tasks focused on factual outcomes~\cite{kunzel2019metalearners}. However, these methods often struggle with treatment selection bias, resulting in biased HTE estimations.

Recent methods, such as counterfactual regression, have shown potential for mitigating selection bias by minimizing distribution discrepancies in the representation space~\citep{cfr, site, disentangle, mim, ace}. However, current methods for discrepancy calculation overlook two critical issues. First, they emphasize a global perspective in calculating distribution discrepancies, neglecting the local proximity between treatment units. Local proximity—where similar units likely exhibit similar outcomes—is a pivotal factor in accurate HTE estimation~\cite{knn, psm, wager2018estimation}. Ignoring this aspect can lead to misleading discrepancy estimates and consequently erroneous updates to the HTE estimator. The second challenge pertains to the curse of dimensionality, where a substantial number of units is required to reliably estimate treatment effects. Often, acquiring a sufficiently large sample of treated units is impractical in real-world settings, rendering the discrepancy estimation unreliable. Addressing these limitations is essential for advancing the precision and applicability of HTE estimations from observational data.

In this work, we propose an effective HTE estimator, namely \textbf{Pro}ximity-enhanced \textbf{C}ounterFactual \textbf{R}egression (CFR-Pro), which tackles both local proximity and dimensionality issues through a generalized optimal transport problem. Specifically, to incorporate local proximity, CFR-Pro incorporates a pairwise proximity regularizer (PPR) in the optimal transport formulation to explicitly maintain local proximity in discrepancy calculation.
To mitigate the curse of dimensionally, CFR-Pro innovatively introduces an informative subspace projector (ISP), which seeks for an informative subspace to calculate distribution discrepancy with minimal precision loss. 
The architecture and computation workflow of CFR-Pro are detailed in Section~\ref{sec:archi}. Extensive experimental results demonstrate that CFR-Pro accurately matches units across varying treatment groups, effectively mitigates treatment selection bias, and significantly outperforms its competitors.

\textbf{Contributions.} The contributions are summarized as follows.
\begin{itemize}[leftmargin=*]
    \item We innovatively investigate the local proximity preservation and curse of dimensionality issues for causal balancing, which historically limits the performance of HTE estimation based on representation learning.
    \item We propose CFR-Pro, a streamelined approach that employs a pairwise proximity regularizer and an informative subspace projector within a unified optimal transport framework to overcome the above issues.
    \item Through comprehensive experiments on open benchmarks, we demonstrate that CFR-Pro outperforms a range of competitors. We further substantiate its effectiveness via extensive hyperparameter tuning and ablative studies.
\end{itemize}


\section{Preliminaries}\label{sec:preliminary}
\begin{figure*}
\centering
\subfigure[Mitigating selection bias with $\psi(\cdot)$.]{\includegraphics[width=0.38\linewidth]{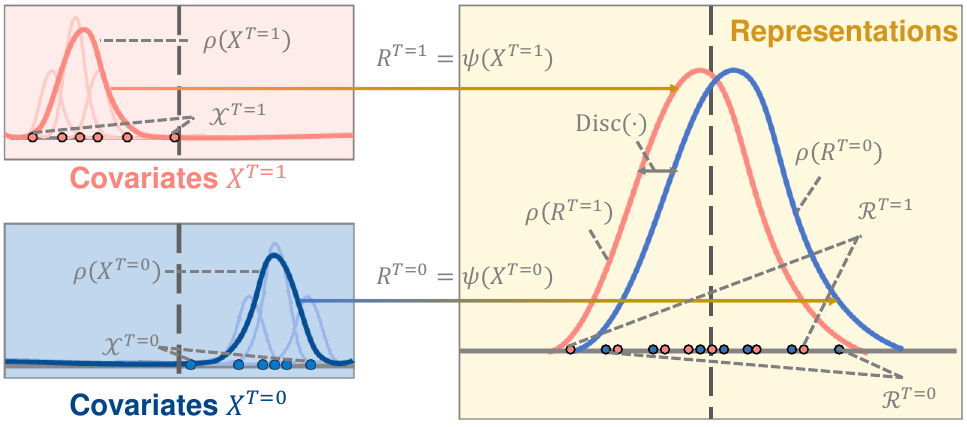}\label{fig:problem}}\quad
\subfigure[Overall architecture of CFR-Pro.]{\includegraphics[width=0.55\linewidth]{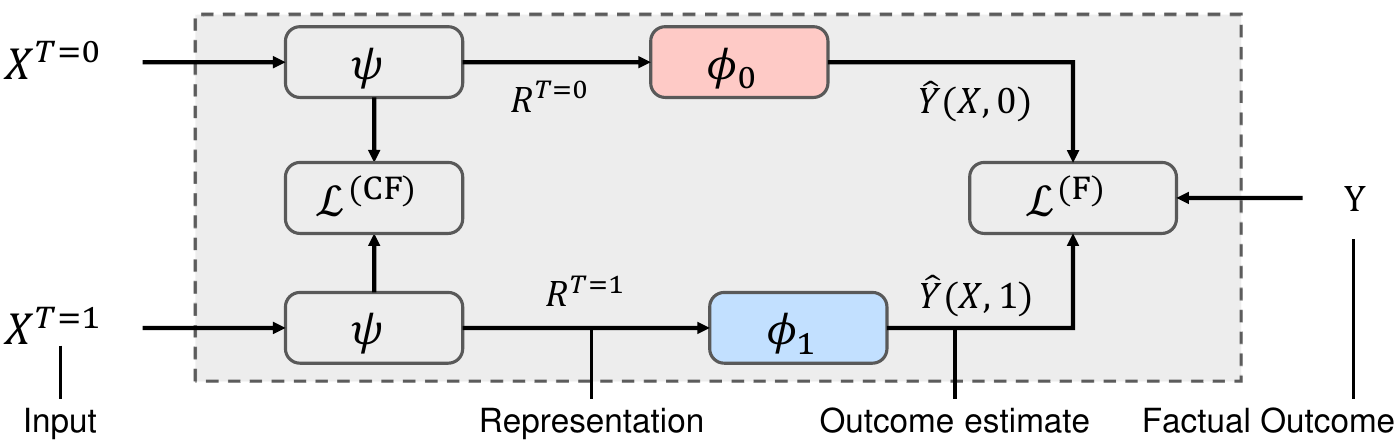}\label{fig:structure}}
\caption{Overview of handling treatment selection bias with CFR-Pro. The red and blue colors signify the treated and untreated groups, respectively.
(a) The treatment selection bias is illustrated through a distribution shift between treated ($X_1$) and untreated ($X_0$) units. The curves and scatters indicate the probability density functions and associated empirical distributions, respectively.
(b) CFR-Pro reduces selection bias by aligning units from both treatment groups within a common representation space, denoted as $R = \psi(X)$. This alignment facilitates the generalization of the outcome mappings $\phi_1$ and $\phi_0$ across different groups.}
\end{figure*}

\subsection{HTE estimation with observational data}\label{sec:causal}
This section outlines the HTE estimation task within the potential outcome framework~\citep{pof} and the challenge of treatment selection bias. The fundamental encapsulated in Definition~\ref{def:rv}\footnote{We use uppercase letters, e.g., $X$, to denote a random variable, and lowercase letters, e.g., $x$, to denote a specific value. Letters in calligraphic font, e.g., $\mathcal{X}$, represent the empirical distribution, and $\mathbb{P}()$ represents the probability distribution function, e.g., $\mathbb{P}(X)$.}. Specifically, a unit characterized by covariates $x$ possesses two potential outcomes: $Y_1$ if treated and $Y_0$ if untreated. The expected difference between these potential outcomes given covariates, represented as $\tau(x) = \mathbb{E}[Y_1 - Y_0 \mid x]$, is termed as the conditional average treatment effect (CATE), and its expectation over all units is termed as the average treatment effect (ATE).

\begin{definition}\label{def:rv}
Suppose $X$, $R$, $Y$, and $T$ are random variables with probability density function $\rho_*$ and support $\mathcal{S}_*$. Typically, $X$ represents covariates with the probability density function $\rho_X$ and support $\mathcal{S}_X=\{0,1\}$. $R$ represents induced representations, $Y$ denotes outcomes, and $T$ denotes treatment indicators. 
\end{definition}

\begin{definition}\label{def:map}
    Suppose $\psi:\mathcal{S}_X\rightarrow\mathcal{S}_R$ is a representation mapping with $R=\psi(X)$. Define $\phi_T:\mathcal{R}\times\mathcal{T}\rightarrow\mathcal{Y}$ as an outcome mapping that maps the representations and treatment to the corresponding factual outcome: $Y_1=\phi_1(R)$ and $Y_0=\phi_0(R)$.
\end{definition}

\begin{definition}\label{def:empirical}
    Suppose $\mathcal{X}$, $\mathcal{R}$, and $\mathcal{Y}$ are the empirical distributions of $X$, $R$, and $Y$ at a minibatch level, respectively. Let $\mathcal{X}^{T=1}$ and $\mathcal{X}^{T=0}$ be the covariates of treated and untreated units, respectively, with $\mathcal{R}_{\psi}^{T=1}$ and $\mathcal{R}_{\psi}^{T=0}$ as the corresponding representations induced by $R=\psi(X)$.
\end{definition}

The estimation of CATE is the cornerstone in HTE estimation. Since one of these outcomes is always unobserved in a dataset, effective CATE estimation typically involves decomposing it into factual outcome estimation subproblems solvable with supervised regression methods~\citep{kunzel2019metalearners}.
An exemplary approach TARNet~\citep{cfr} employs the representation mapping $\psi$ and outcome mapping $\phi$ from Definition~\ref{def:map}, which estimates the CATE as
\begin{equation}\begin{aligned}
    \hat{\tau}_{\psi,\phi}(x)&=\hat{Y}(x,1)-\hat{Y}(x,0),\\
    \hat{Y}(x,1)&=\phi_1(\psi(x)), \quad \hat{Y}(x,0)=\phi_0(\psi(x)),
\end{aligned}\end{equation}
where $\psi$ is trained across all units, while $\phi_1$ and $\phi_0$ are trained separately on treated and untreated groups. The training objective is to minimize the factual outcome estimation error:
\begin{equation}\label{eq:factual}
    \mathcal{L}^{(\mathrm{F})}(\psi,\phi):=\sum_{i=1}^\mathrm{N}\left\|\phi_1(\mathcal{R}_{\psi,i}^{T=1})-\mathcal{Y}_i^{T=1}\right\|_2^2+    \sum_{j=1}^\mathrm{M}\left\|\phi_0(\mathcal{R}_{\psi,j}^{T=0})-\mathcal{Y}_j^{T=0}\right\|_2^2,
\end{equation}
where $\mathcal{R}_\psi$ and $\mathcal{Y}$ are the empirical distributions of representations and outcomes as defined in Definition~\ref{def:empirical}, and $i$ and $j$ are sample indices in the associated empirical distribution. CATE estimators are evaluated using the precision in estimation of heterogeneous effect (PEHE) metric:
\begin{equation}\label{eq:pehe}
    \epehe(\psi,\phi):=\int\left(\hat{\tau}_{\psi,\phi}(x)-\tau(x)\right)^{2} {\color{black}\rho(x)}\,dx.
\end{equation}

\paragraph{Selection bias.} As illustrated in Figure~\ref{fig:problem}, treatment selection bias introduces a distribution shift of covariates between groups, which causes $\phi_1$ and $\phi_0$ to overfit to their respective group's characteristics and generalize poorly across the entire population. To mitigate selection bias, seminal works starting from CFR~\cite{cfr} augment the learning objective with a \textit{distribution discrepancy} term as $\mathcal{L}^{(\mathrm{F})}+\mathrm{Disc}(\mathcal{R}_{\psi}^{T=1}, \mathcal{R}_{\psi}^{T=0})$. This adjustment reduces distribution shift in the representation space, thereby enabling $\phi_1$ and $\phi_0$ to generalize to both treated and untreated groups.

\subsection{Local proximity for HTE estimation}

Local proximity, quantified as the mutual distance between units within a distribution, encapsulates the geometric properties of distributions.  The assumption that \textit{\underline{similar units have similar outcomes}}~\cite{site,ace} highlights its critical role in HTE estimation. This principle is central to various HTE estimators—such as matching techniques (e.g., KNN~\cite{knn}, propensity score matching~\cite{psm}) and stratification methods~\cite{wager2018estimation}—that leverage proximity to enhance estimation accuracy.
Despite its acknowledged importance, modern HTE approaches, particularly those based on the Counterfactual Representation (CFR) paradigm, primarily focus on minimizing a global discrepancy metric, denoted as  $\mathrm{Disc}(\cdot)$, while neglecting the nuances of local proximity that can be pivotal for precise causal inference.

One notable exception is the SITE model~\cite{site}, which incorporates the PDDM metric~\cite{pddm} to measure proximity. SITE differs fundamentally from our work in two respects. Firstly, SITE employs PDDM merely to align latent and covariate spaces while using a simplistic middle-point distance as its discrepancy measure, thereby not integrating local proximity into the discrepancy for balancing treated and untreated samples. In contrast, we focus on enhancing the discrepancy by comprehensively incorporating local proximity. Secondly, SITE quantifies proximity using only six pre-selected anchor units, which may limit its ability to capture broader contextual information. Our approach transcends this limitation by leveraging an optimal transport methodology to construct a more inclusive and nuanced representation of local proximity.

\subsection{Discrete optimal transport}\label{sec:ot}
Optimal transport (OT) quantifies distribution discrepancy as the minimum transport cost~\cite{otda,pswi,wang2025tois}, offering a tool to quantify the treatment selection bias in Figure~\ref{fig:problem}.
An applicable formulation proposed by \citet{kantorovich2006translocation} is present in Definition~\ref{def:kanto}, which can be seen as a linear programming problem.

\begin{definition}\label{def:kanto}
    For empirical distributions $\alpha$ and $\beta$ with n and m units, respectively, the Kantorovich problem aims to find a feasible plan $\pi\in\mathbb{R}_{+}^{n\times m}$ which transports $\alpha$ to $\beta$ at the minimum cost:
    \begin{equation}\label{eq:kanto}
    \begin{aligned}
        \mathbb{W}(\alpha,\beta)&:=\min_{\bpi\in\Pi(\alpha,\beta)}\left<\mathbf{D},\mathbf{\bpi}\right>,\;\\
        \Pi(\alpha,\beta)&:=\left\{\mathbf{\bpi}\in\mathbb{R}_{+}^{n\times m}: \bpi\mathbf{1}_m=\mathbf{a},\bpi^\mathrm{T}\mathbf{1}_n=\mathbf{b}\right\},
    \end{aligned}  
    \end{equation}
    where $\mathbb{W}(\alpha,\beta)\in\mathbb{R}$ is the OT discrepancy;
    $\mathbf{D}\in\mathbb{R}_{+}^{n\times m}$ is the unit-wise Euclidean distance between $\alpha$ and $\beta$;
    $\mathbf{a}$ and $\mathbf{b}$ are the mass of units in $\alpha$ and $\beta$, respectively;
    $\Pi$ is the feasible plan set where the mass-preserving constraint holds.
\end{definition}


\section{Proposed method}\label{sec:proposed}

In this section, we present the Proximity-enhanced CounterFactual Regression (CFR-Pro) approach, which leverages OT to tackle the treatment selection bias.
We first illustrate the pair-wise proximity regularizer (PPR) for measuring and maintaining local proximity in different treatment groups, and demonstrate its efficacy for improving HTE estimation. 
Subsequently, we propose an informative subspace projector (ISP) to reduce the sampling complexity and handle the curse of dimensionality in calculating OT. We finally open a new thread to summarize the model architecture, learning objectives, and optimization algorithm.

\subsection{Pair-wise proximity regularizer for the preservation of local proximity}\label{sec:sot}
To mitigate treatment selection bias, representation-based methods align treated and untreated groups in the representation space, the core of which is the quantification of the distribution discrepancy $\mathrm{Disc}(\cdot)$ between treatment groups.
It is plausible to quantify the discrepancy with OT due to its numerical advantages and flexibility over competitors~\cite{escfr}.
However, standard OT overlooks local proximity, a crucial aspect in HTE estimation. The treated and untreated units with similar neighbors for instance should have a higher probability of matching together since similar units have similar outcomes~\citep{psm,wager2018estimation}. 

An extension of OT that encodes local proximity is the Gromov-Wasserstein measure, primarily applied to matching objects with geometric structures~\cite{gromov-object,gromov-graph}.
On the basis, inspired by~\citep{fgw}, the PPR fuses the Gromov Wasserstein measure and restates the transport problem in the representation space as:
\small{\begin{equation}\label{eq:fgw}
        \mathbb{F}(\mathcal{R}_\psi^{T=0}, \mathcal{R}_\psi^{T=1}):=\min_{\bpi\in\Pi(\alpha,\beta)} \left( \kappa \langle \bpi, \mathbf{D}\rangle+ (1-\kappa) \sum_{i, j, k, l}\mathbf{P}_{i,j,k,l} \bpi_{i, j} \bpi_{k, l}\right),
\end{equation}}\normalsize
where $0\leq\kappa\leq1$ controls the relative strength. The first term, following the standard OT formulation in~\eqref{eq:kanto}, measures the global discrepancy between treatment groups with $\mathbf{D}_{i,j}=\left\|\mathcal{R}_{\psi,i}^{T=0}-\mathcal{R}_{\psi,j}^{T=1}\right\|_2^2$.
The second term measures local proximity within each treatment group as $\mathbf{D}_{i,j}^{t}=\left\|\mathcal{R}_{\psi,i}^{T=t}- \mathcal{R}_{\psi,j}^{T=t}\right\|_2^2$, and incorporates such local proximity via $\mathbf{P}_{i,j,k,l}=\left\|\mathbf{D}_{i,k}^{T=0}- \mathbf{D}_{j, l}^{T=1}\right\|_2$. 
Specifically, if the distance between $\mathcal{R}_{\psi,i}^{T=0}$ and $\mathcal{R}_{\psi,k}^{T=0}$ is close to that between $\mathcal{R}_{\psi,j}^{T=1}$ and $\mathcal{R}_{\psi,l}^{T=1}$ (i.e., $\left\|\mathbf{D}_{i,k}^{T=0}- \mathbf{D}_{j, l}^{T=1}\right\|_2^2 \rightarrow 0$), a higher volume of mass will be matched, indicated by a larger $\bpi_{i, j} \bpi_{k, l}$. Conversely, less mass will be transported.  The derived transport plan encourages matching units with similar neighbors, preserving local proximity. Therefore, $\mathbb{F}$ quantifies the discrepancy between treatment groups while accommodating the preservation of local proximity.

\begin{figure}
\centering
\begin{center}
\centerline{\includegraphics[width=\linewidth]{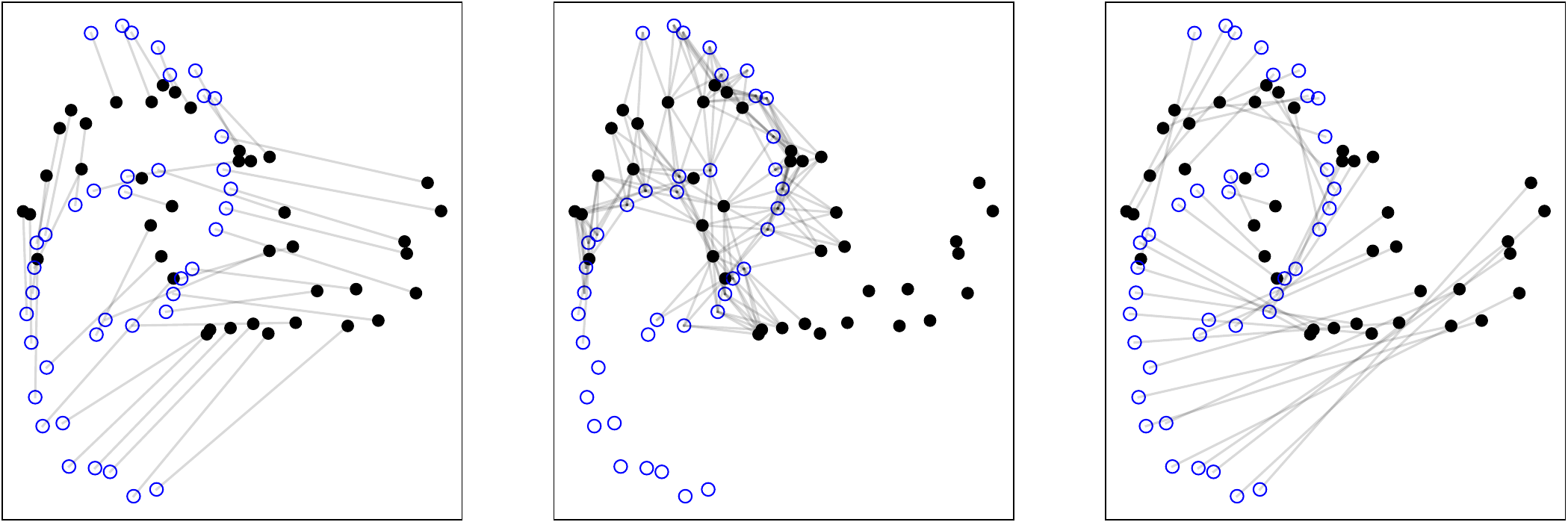}}
\centerline{\includegraphics[width=\linewidth]{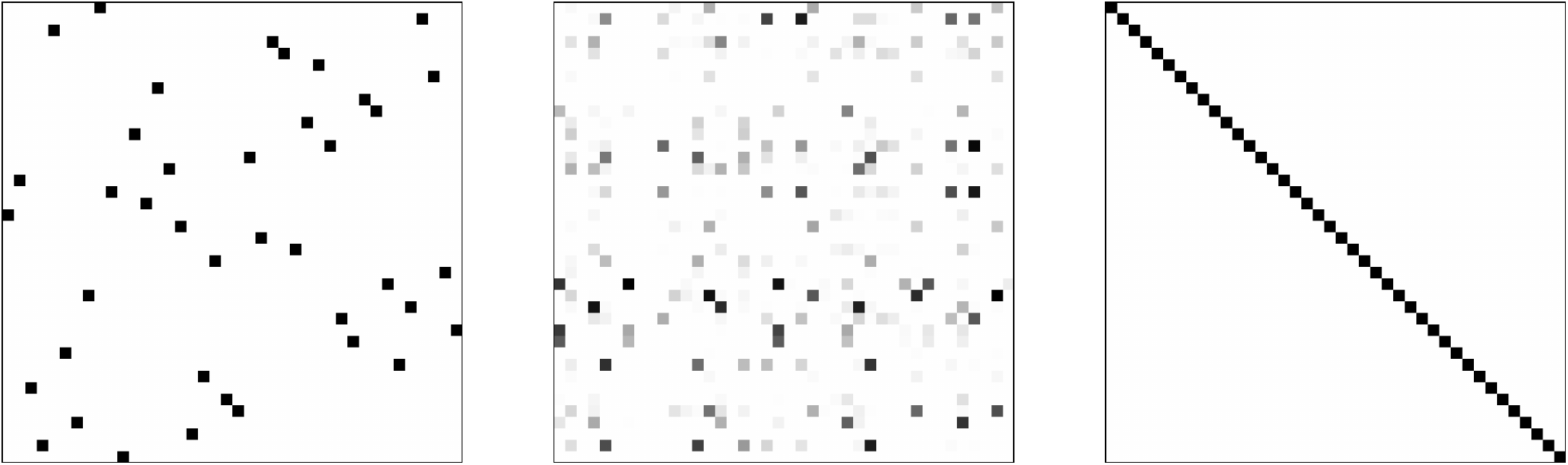}}
\caption{Overview of the transport strategies in three HTE estimators: CFR~\cite{cfr} (left), ESCFR~\cite{escfr} (center) and Ours (right). The upper panels visualize the sample locations of two different treatment groups, where different scatter colors indicate different treatments. The generated transport strategies are marked with gray lines. The down panels elaborate on the generated transport strategies by visualizing the transport matrices $\bpi$. The darkness indicates the mass transported.}
\label{fig:vis}
\end{center}
\end{figure}

\textbf{Case study.} To showcase the importance of preserving local proximity, a toy dataset is provided in Figure~\ref{fig:vis}. The untreated group units are simulated from a two-moon distribution (black circles) and treated group units are generated by rotating these positions 90 degrees (blue circles) and adding a small horizontal shift. This setup mimics shifts in both mean and covariance across treatment groups. The ideal transport strategy should align pre- and post-rotation samples, manifested by a primarily diagonal transport matrix.
We compare the strategies used in CFR~\cite{cfr}, ESCFR~\cite{escfr} and the PPR-enhanced CFR. Key observations are presented below.
\begin{itemize}[leftmargin=*]
    \item The canonical Wasserstein discrepancy in CFR~\cite{cfr} targets global alignment between groups. However, it could cause erroneous matching as it does not account for local proximity, which that can mislead the update of HTE estimators.
    \item The unbalanced Wasserstein discrepancy in ESCFR~\cite{escfr} discards skewed samples and concentrates on aligning units that overlap between treatment groups. Despite improvements, this strategy can still produce a noisy and blurred transportation of units.
    \item Our method integrates PPR to exploit local proximity, effectively rectifying the transport strategy and ensuring precise matching of all samples. This property is essential for generating accurate gradient signals to update representation mappings.
\end{itemize}

\textbf{Theoretical investigation.} Although PPR has shown practical efficacy, questions remain regarding its contribution to minimizing PEHE. Theorem~\ref{thm:bound} (refer to Appendix~\ref{apdx_thm} for proof) offers a theoretical bound, which demonstrates that PEHE can be optimized by minimizing the estimation error of factual outcomes and the group discrepancy with PPR term. Notably, the integration of PPR slightly expands the theoretical bound compared to that of canonical CFR, yet it is promising to trade off some tightness of the bound for the preservation of local proximity, due to its importance to produce a viable transport strategy, as demonstrated in Figure~\ref{fig:vis}.

\begin{thm}\label{thm:bound}
    Let $\psi$ and $\phi$ be the representation mapping and factual outcome mapping, respectively;
    $\hat{\mathbb{W}}_\psi$ be the group discrepancy at a mini-batch level.
    With the probability of at least $1-\delta$, we have:
    \begin{equation}\label{eq:peheBound}
    \begin{aligned}
        \epehe(\psi,\phi)
        &\leq 2[
            \epsilon^{T=1}_\mathrm{F}(\psi,\phi) + \epsilon^{T=0}_\mathrm{F}(\psi,\phi) + B_{\psi, \kappa} \mathbb{F}(\mathcal{R}_\psi^{T=0}, \mathcal{R}_\psi^{T=1})] \\
            &-4\sigma^2_{Y}+\mathcal{O}({N}^\mathrm{-\frac{2}{\mathrm{d}}}),
    \end{aligned}
    \end{equation}
    where $\epsilon^{T=1}_\mathrm{F}$ and $\epsilon^{T=0}_\mathrm{F}$ are the expected errors of factual outcome estimation,
    $N$ is the batch size, $\sigma^2_Y$ is the variance of outcomes, 
    $B_{\psi, \kappa}$ is a constant term, and $\mathcal{O}(\cdot)$ is a sampling complexity term.
\end{thm}

\subsection{Informative subspace projector for the curse of dimensionality}

The curse of dimensionality refers to the phenomenon where the Euclidean distance between data points tend to be identical~\citep{concen}.
It renders Euclidean distance difficult to model the proximity in high dimensional settings due to diminished discrimination.

From a computational view, the diminishing discrimination necessitates more samples for estimating Euclidean-based discrepancy~\cite{prw}, \textit{which is a pivotal component in state-of-the-art HTE estimators} such as MMD in \cite{cfr}, PDDM in \cite{site}, and EMD in \cite{escfr}. A well-known result states, for instance, that the sample complexity of EMD can grow exponentially with dimension~\cite{otbound}. Similarly, the sample complexity of 
$\mathbb{F}$ reaches $\mathcal{O}({N}^{-\frac{2}{d}})$, forming the complexity term in Theorem 3.1. 
Such large sample complexity necessitates many treated and untreated units to faithfully estimate the true discrepancy. However, a large number of treated units is often difficult to acquire in real-world experiments, underscoring the adverse impact of the curse of dimensionality on HTE estimation.

To handle the curse of dimensionality, a common strategy involves reducing the dimension of the computational space. A naive approach is reducing the hidden dimension of the representation mapping \(\psi\). However, it can be counterproductive as it may limit its capacity, preventing the mapping from capturing complex patterns and nonlinearities in data. In this study, we propose identifying an informative subspace and calculating pairwise distances within this subspace, thus mitigating the curse of dimensionality while preserving the full capacity of \(\psi\).
Specifically, given a projector $U\in\mathbb{R}^{d\times k}$ to transform the data from a high-dimensional space $d$ to a lower-dimensional subspace $k$. The distance between two unit representations $\mathcal{R}_{i}$ and $\mathcal{R}_{j}$ can be computed as $\mathbf{D}_{i,j}^P=\left\|\mathcal{R}_1U-\mathcal{R}_2U\right\|$. This approach alleviates the curse of dimensionality but introduces the risk of losing significant information, potentially leading to overly optimistic discrepancy estimations.

The central problem is how to find the informative subspace projector. To reduce the information loss caused by naive dimension reduction, it is feasible to determine the projector in a adversarial manner as follow:
\small{\begin{equation*}
        \min_{\bpi\in\Pi(\mathcal{R}_\psi^{T=0}, \mathcal{R}_\psi^{T=1})} \max_{UU^\top=I}\left( \kappa\cdot \langle \bpi, \mathbf{D}^U\rangle+ (1-\kappa)\cdot \sum_{i, j, k, l} \mathbf{P}_{i,j,k,l}^U \bpi_{i, j} \bpi_{k, l}\right),
\end{equation*}}\normalsize
where $\textbf{D}_{i,j}^U=\left\|\mathcal{R}_{\psi,i}^{T=0}U-\mathcal{R}_{\psi,j}^{T=1}U\right\|_2^2$ is the distance measured in the reduced k-dimensional space; $\mathbf{P}_{i,j,k,l}^U=\left\|\mathbf{D}_{i,k}^{U,T=0}- \mathbf{D}_{j, l}^{U,T=1}\right\|_2$.
However, this optimization problem proves difficult to solve~\cite{prw}. An effective compromise involves projecting the data into a k-dimensional subspace through an ISP module that maximally preserves the information, and then compute discrepancy in the subspace. Built upon this idea, the transport problem modified with the ISP module is formulated in Definition~\ref{def:robust}.

\begin{definition}\label{def:robust}
    Suppose $U^*\in\mathbb{R}^{d\times k}$ is an informative subspace projector that is obtained by:
    \begin{equation}\label{eq:dimr}
        U^*=\arg\min_{U, UU^\top=I} \left\|\mathcal{R}-\mathcal{R}UU^\top\right\|_2^2,
    \end{equation}
    where $\mathrm{P}=k/d$ denotes the ratio of dimensionality reduction. The distribution discrepancy equipped with PPR and ISP modules is formulated as \small{
    \begin{equation}
        \mathbb{P}_{\kappa, \mathrm{P}}(\psi):=\min_{\bpi\in\Pi(\mathcal{R}_\psi^{T=0}, \mathcal{R}_\psi^{T=1})} \left( \kappa\cdot \langle \bpi, \mathbf{D}^{U^*}\rangle+ (1-\kappa)\cdot \sum_{i, j, k, l} \mathbf{P}_{i,j,k,l}^{U^*} \bpi_{i, j} \bpi_{k, l}\right).
    \end{equation}}\normalsize
\end{definition}

\subsection{Overall workflow of CFR-Pro}\label{sec:archi}

The architecture of CFR-Pro is presented in Figure~\ref{fig:structure}, where the covariate $X$ is first mapped to the representations $R$ with $\psi(\cdot)$, and then to the potential outcomes with $\phi(\cdot)$.
The learning objective is to minimize the risk of factual outcome estimation and the group discrepancy.
Given mini-batch distributions $\mathcal{X}^{T=1}$ and $\mathcal{X}^{T=0}$, the risk of factual outcome estimation can be estimated as \eqref{eq:factual}. Afterwards, the group discrepancy is calculated as $\mathbb{P}_{\kappa, \mathrm{P}}(\psi)$. 
Finally, the overall learning objective of CFR-Pro is
\begin{equation}\label{eq:escfr}
    \mathcal{L}^{(\mathrm{CFR-Pro})}_{\lambda, \kappa, \mathrm{P}} := \mathcal{L}^{(\mathrm{F})}(\psi,\phi) + \lambda\cdot \mathbb{P}_{\kappa, \mathrm{P}}(\psi),
\end{equation}
where $\lambda$ controls the strength of distribution alignment,
$\kappa$ controls PPR in \eqref{eq:fgw},
and $\mathrm{P}$ controls the ratio of dimension reduction in Definition~\ref{def:robust}.
The learning objective above mitigates the selection bias following Theorem~\ref{thm:bound} while handling the issues of local proximity and curse of dimensionality.

The optimization procedure of CFR-Pro is encapsulated in Algorithm 1. 
First, we compute the latent space representations using the representation mapping $\psi$.
Second, we determine the informative subspace projector $U^*$ by solving the dimension reduction problem in \eqref{eq:dimr}, which can be solved via the well-established principal component analysis.
Then, we calculate the pair-wise distance matrix $\mathbf{D}^U$ and $\mathbf{P}^U$ in the subspace induced by $U^*$.
Subsequently, compute the distribution discrepancy term $\mathbb{P}$ by solving the OT problem in Definition~\ref{def:robust}. The solution process is available in \citep{fgw}.
Finally, we compute the overall loss in \eqref{eq:escfr} and update $\psi$ and $\phi$ with stochastic gradient methods.

\begin{algorithm}[tb]
\caption{The computation workflow of CFR-Pro.}\label{alg:escfr}
\leftline{\textbf{Input}: covariates $\mathcal{X}$; factual outcomes $\mathcal{Y}$; outcome mapping $\phi$;}
\leftline{treatments $\mathcal{T}$; representation mapping $\psi$.}

\leftline{\textbf{Parameter}: $\lambda$: strength of discrepancy alignment; $\kappa$: strength of}
\leftline{proximity preservation in PPR; $\mathrm{P}$: ratio of dimension reduction.} 
\leftline{$\mathrm{B}$: batch size}
\leftline{\textbf{Output}: $\mathcal{L}^{(\mathrm{CFR-Pro})}_{\lambda, \kappa, \mathrm{P}}$: the learning objective of CFR-Pro.}
\begin{algorithmic}[1] 
    \State
    $\mathcal{R}\gets\psi(\mathcal{X})$.
    \State
    $U^*=\arg\min_{U, UU^\top=I} \left\|\mathcal{R}-\mathcal{R}UU^\top\right\|_2^2$.
    \State
    $\mathbf{D}_{i,j}^{U^*} \gets \Vert\mathcal{R}_i^{T=0}U^*-\mathcal{R}_j^{T=1}U^*\Vert_2^2$ \quad for $1\leq i,j \leq \mathrm{B}$.
    \State
    $\mathbf{D}_{i,k}^{U^*,T=0} \gets \Vert\mathcal{R}_i^{T=0}U^*-\mathcal{R}_k^{T=0}U^*\Vert_2^2$ \quad for $1\leq i,k \leq \mathrm{B}$.
    \State
    $\mathbf{D}_{j,l}^{U^*,T=1} \gets \Vert\mathcal{R}_j^{T=1}U^*-\mathcal{R}_l^{T=1}U^*\Vert_2^2$ \quad for $1\leq j,l \leq \mathrm{B}$.
    \State
    $\mathbf{P}_{i,j,k,l}^{U^*}\gets\left\|\mathbf{D}_{i,k}^{U^*,T=0}- \mathbf{D}_{j, l}^{U^*,T=1}\right\|_2$  \quad for $1\leq i,j,k,l \leq \mathrm{B}$.
    \State Calculate
    $\mathbb{P}_{\kappa, \mathrm{P}}(\psi)$ following Eq. (8).
    \State
    $\mathcal{L}^{(\mathrm{F})}(\psi, \phi)\gets
    \Vert\phi(\mathcal{R},\mathcal{T})-\mathcal{Y}\Vert_2^2$.
    \State
    $\mathcal{L}^{(\mathrm{CFR-Pro})}_{\lambda, \kappa, \mathrm{P}} \leftarrow \mathcal{L}^{(\mathrm{F})}(\psi,\phi) + \lambda\cdot \mathbb{P}_{\kappa, \mathrm{P}}(\psi)$.
\end{algorithmic}
\end{algorithm}

\section{Experiments}\label{sec:experiments}
We validate CFR-Pro by investigating the aspects as follows:
\begin{enumerate}[leftmargin=*]
    \item \textbf{Performance:} \textit{Does CFR-Pro work?} Section \ref{sec:overall} compares CFR-Pro against established CATE estimators on real-world public benchmarks, where CFR-Pro achieves the best performance.
    \item \textbf{Efficacy:} \textit{How does it work?} Section \ref{sec:ablation} conducts an ablative study to investigate the contribution of PPR and ISP, where both components are beneficial to improve canonical CFR and cooperate well.
    \item \textbf{Sensitivity:} \textit{Is it sensitive to hyperparameters?} Section~\ref{sec:hyper1}  conducts a sensitivity study on the hyperparameters introduced by PPR and ISP, respectively, and give further insights on the rational they work.
\end{enumerate}
\subsection{Experimental setup}

\begin{table*}
\caption{Out-of-sample performance (mean±std) on the ACIC and IHDP datasets. ``*'' marks the baseline estimators that CFR-Pro outperforms significantly at p-value $<$ 0.05 over paired samples t-test.}\label{tab:result}
\renewcommand{\arraystretch}{0.9} 
\setlength{\tabcolsep}{8pt}
\begin{tabular}{lccccccccc}
\toprule
Dataset    & \multicolumn{3}{c}{ACIC}    && \multicolumn{3}{c}{IHDP} \\ \cmidrule{2-4} \cmidrule{6-8} 
Metric     & $\epsilon_\mathrm{PEHE}$ & $\epsilon_\mathrm{ATE}$ & $\epsilon_\mathrm{ATT}$ && $\epsilon_\mathrm{PEHE}$ & $\epsilon_\mathrm{ATE}$ & $\epsilon_\mathrm{ATT}$\\ \midrule
R.Forest       & 3.3908$_{\pm 0.1811}$\sig   & 0.8347$_{\pm 0.3635}$  & 0.7785$_{\pm 0.3816}$  && 4.6697$_{\pm 9.2920}$     & 0.4544$_{\pm 0.8308}$   & 0.8353$_{\pm 1.2413}$  \\
S.Learner       & 4.8835$_{\pm 0.7933}$\sig   & 3.0913$_{\pm 0.7731}$\sig  & 3.1213$_{\pm 0.6372}$\sig  && 4.7408$_{\pm 3.9688}$\sig     & 2.5785$_{\pm 1.8521}$\sig   & 2.7951$_{\pm 1.6036}$\sig  \\
T.Learner       & 4.2749$_{\pm 0.6793}$\sig   & 2.2176$_{\pm 1.2131}$\sig  & 2.3940$_{\pm 1.1964}$\sig  && 2.5257$_{\pm 3.3643}$\sig     & 0.5818$_{\pm 1.2177}$   & 0.6642$_{\pm 1.2197}$  \\
TARNet       & 3.5331$_{\pm 0.9556}$\sig   & 1.5308$_{\pm 1.0469}$\sig  & 1.6601$_{\pm 1.0670}$\sig  && 1.7781$_{\pm 3.4467}$     & 0.2814$_{\pm 0.3193}$   & 0.3338$_{\pm 0.3349}$  \\
DESCN       & 2.6420$_{\pm 0.2614}$\sig   & 0.4548$_{\pm 0.1693}$  & 0.4987$_{\pm 0.1881}$  && 4.0128$_{\pm 6.1409}$\sig     & 1.2219$_{\pm 1.7453}$   & 0.7917$_{\pm 1.3185}$  \\
\midrule
k-NN       & 5.8977$_{\pm 0.1400}$\sig   & 1.5773$_{\pm 0.3075}$\sig  & 1.9068$_{\pm 0.2870}$\sig  && 4.3191$_{\pm 7.3361}$     & 0.8316$_{\pm 1.6911}$   & 1.8118$_{\pm 3.2342}$ \\
O.Forest       & 2.7451$_{\pm 0.3379}$\sig   & 0.6003$_{\pm 0.1879}$  & 0.6597$_{\pm 0.2013}$  && 3.1888$_{\pm 5.6657}$     & 0.3150$_{\pm 0.3696}$   & 0.6539$_{\pm 0.5370}$\sig  \\
PSM       & 5.1014$_{\pm 0.2987}$\sig   & 0.6468$_{\pm 0.3478}$  & 0.6231$_{\pm 0.3465}$  && 4.6347$_{\pm 8.5748}$     & 0.2129$_{\pm 0.3362}$   & 0.9353$_{\pm 2.7094}$  \\
\midrule
CFR-MMD       & 3.8514$_{\pm 0.4558}$\sig   & 1.7379$_{\pm 0.9133}$\sig  & 1.9060$_{\pm 0.9290}$\sig  && 1.9398$_{\pm 2.9029}$     & 0.5870$_{\pm 1.2231}$   & 0.6678$_{\pm 1.2207}$  \\
CFR-WASS       & 3.3187$_{\pm 0.7622}$\sig   & 1.3581$_{\pm 1.0325}$\sig  & 1.4682$_{\pm 1.0636}$\sig  && 1.9252$_{\pm 2.9323}$     & 0.5578$_{\pm 1.2455}$   & 0.6532$_{\pm 1.2544}$  \\
SITE      & 3.4910$_{\pm 0.7799}$\sig   & 1.3425$_{\pm 1.1929}$  & 1.5443$_{\pm 1.2128}$\sig  && 1.7339$_{\pm 3.1709}$     & 0.2271$_{\pm 0.3140}$   & 0.2525$_{\pm 0.2805}$  \\
ESCFR       & 2.6780$_{\pm 0.6566}$   & 1.1468$_{\pm 0.8146}$\sig  & 1.2365$_{\pm 0.8689}$\sig  && 1.6299$_{\pm 3.0344}$     & 0.2135$_{\pm 0.3788}$   & 0.2319$_{\pm 0.2488}$  \\
\midrule
CFR-Pro       & \textbf{2.0413}$_{\pm 0.6646}$   & \textbf{0.4551}$_{\pm 0.3845}$  & \textbf{0.5034}$_{\pm 0.4221}$  && \textbf{1.4601}$_{\pm 2.6607}$     & \textbf{0.1079}$_{\pm 0.1087}$   & \textbf{0.2224}$_{\pm 0.2472}$  \\\bottomrule
\end{tabular}
\end{table*}
\begin{table*}[]
\caption{Ablation study (mean±std) on the ACIC benchmark. ``*'' marks the variants that CFR-Pro outperforms significantly at p-value $<$ 0.01 over paired samples t-test.}\label{tab:ablation}
\centering
\renewcommand{\arraystretch}{1} 
\setlength{\tabcolsep}{6.5pt}
\begin{tabular}{lllccccccccc}
\toprule
    &&& \multicolumn{3}{c}{In-sample} &  & \multicolumn{3}{c}{Out-sample} \\ \cmidrule{4-6} \cmidrule{8-10} 
Model&PPR&ISP      & $\epsilon_\mathrm{PEHE}$   & $\epsilon_\mathrm{ATE}$ & $\epsilon_\mathrm{ATT}$ && $\epsilon_\mathrm{PEHE}$   & $\epsilon_\mathrm{ATE}$ & $\epsilon_\mathrm{ATT}$  \\ \midrule
CFR&\textcolor{lightgray}{\XSolidBrush}&\textcolor{lightgray}{\XSolidBrush}   & 3.4288$_{\pm 0.3952}$\sig   & 1.1796$_{\pm 0.6443}$\sig  & 1.9186$_{\pm 0.8632}$\sig && 3.3187$_{\pm 0.7622}$\sig   & 1.3581$_{\pm 1.0325}$\sig  & 1.4682$_{\pm 1.0636}$\sig    \\
CFR$^\dagger$ & \Checkmark&\textcolor{lightgray}{\XSolidBrush}    &  2.9668$_{\pm 0.9142}$ & 0.9162$_{\pm 0.5930}$ & 1.3961$_{\pm 0.9425}$ && 2.5193$_{\pm 0.7771}$ & 0.9164$_{\pm 0.8203}$ & 1.0020$_{\pm 0.8903}$ \\

CFR$^\ddagger$ & \textcolor{lightgray}{\XSolidBrush}&\Checkmark    &  2.9341$_{\pm 0.7583}$ & 0.7825$_{\pm 0.5363}$ & 1.2473$_{\pm 0.7007}$ &&  2.5983$_{\pm 0.7378}$\sig & 0.8303$_{\pm 0.7695}$ & 0.9080$_{\pm 0.8328}$ \\
CFR-Pro & \Checkmark&\Checkmark  &   \textbf{2.6091$_{\pm 0.7673}$} & 	\textbf{0.5384$_{\pm 0.3932}$} & 	\textbf{1.0313$_{\pm 0.7206}$} &&     	\textbf{2.0413$_{\pm 0.6640}$} & 	\textbf{0.4551$_{\pm 0.3845}$} & 	\textbf{0.5034}$_{\pm 0.4221}$ \\

\bottomrule

\end{tabular}
\end{table*}

\paragraph{Datasets.}
The evaluation of PEHE is challenged by the absence of counterfactuals in observational data. To address this challenge, experiments are conducted using two semi-synthetic datasets: the Infant Health and Development Program (IHDP) and the Atlantic Causal Inference Conference (ACIC) competition data \cite{cfr,site}. The IHDP dataset evaluates the effect of specialist home visits on infants' cognitive development, comprising 747 observations with 25 covariates. The ACIC dataset, derived from the Collaborative Perinatal Project \cite{acic}, includes 4802 observations and 58 covariates. All datasets are randomly shuffled and partitioned into training, validation, and test sets in a 0.7:0.15:0.15 ratio, maintaining the same proportion of treated units across all splits to ensure numerical reliability.
To increase the distinguishability of results, we omit dataset scaling to heighten the impact of selection bias.

\paragraph{Baselines.}
The baselines can be categorized into three groups. (1) \textbf{Direct estimators}: R.Forest~\cite{wager2018estimation}, S.learner~\cite{kunzel2019metalearners}, T.learner~\cite{kunzel2019metalearners}, TARNet~\citep{cfr}, DESCN~\cite{descn}; (2) \textbf{Matching estimators}: PSM~\cite{psm}, k-NN~\cite{knn}, O.Forest~\cite{wager2018estimation}; (3) \textbf{Representation-based estimators}: CFR-MMD~\cite{cfr}, CFR-WASS~\cite{cfr}, ESCFR~\cite{escfr}, and SITE~\cite{site}.

\paragraph{Implementation details.}
CFR-Pro is implemented with a fully connected neural network architecture comprising two hidden layers with 16-16 and 32-32 nodes, respectively. It is trained using the Adam optimizer for a maximum of 400 epochs, with an early stopping patience set to 30 epochs. The learning rate and weight decay parameters are set at $1e^{-3}$ and $1e^{-4}$, respectively. Other optimization settings follow~\citet{adam}. Hyper-parameters are tuned on the validation set within the ranges in Section~\ref{sec:hyper1}, with model performance validation conducted every epoch.


\subsection{Overall performance}\label{sec:overall}

Table~\ref{tab:result} provides a comprehensive comparison of the CFR-Pro framework with various baseline methodologies. Key observations from this comparative analysis are outlined below:

\begin{itemize}[leftmargin=*]
    \item Direct estimators exhibit strong performance on the PEHE metric. Neural network-based estimators particularly excel, surpassing linear models and random forests due to their enhanced ability to capture nonlinear relationships. Among them, TARNet, which integrates the strengths of both T-learner and S-learner, achieves superior and stable performance on both datasets. However, the limitations in addressing treatment selection bias hamper their performance in certain scenarios.
    
    \item Matching methods such as PSM and O.Forest show robust capabilities in estimating average treatment effects, which contributes to their widespread adoption in policy evaluation contexts. However, their efficacy diminishes on the PEHE metric, restricting their suitability for applications requiring personalized treatments (e.g., advertising).
    
    \item Representation-based methods effectively address treatment selection bias and enhance HTE estimation performance. However, their oversight of local proximity and curse of dimensionality restricts their effectiveness in overcoming selection bias.
    \item  CFR-Pro surpasses other prevalent baselines with significant improvements across most metrics. This superiority is attributed to the innovative PPR and ISP modules. These components cooperate to enable CFR-Pro to adeptly harness the local proximity and mitigate the curse of dimensionality, which facilitates more accurate alignment of treatment groups and thereby handling of selection bias.
\end{itemize}

\subsection{Ablation study}\label{sec:ablation}
\begin{figure*}
    \centering
    \subfigure[Performance with varying $\lambda$]{\includegraphics[width=0.49\linewidth]{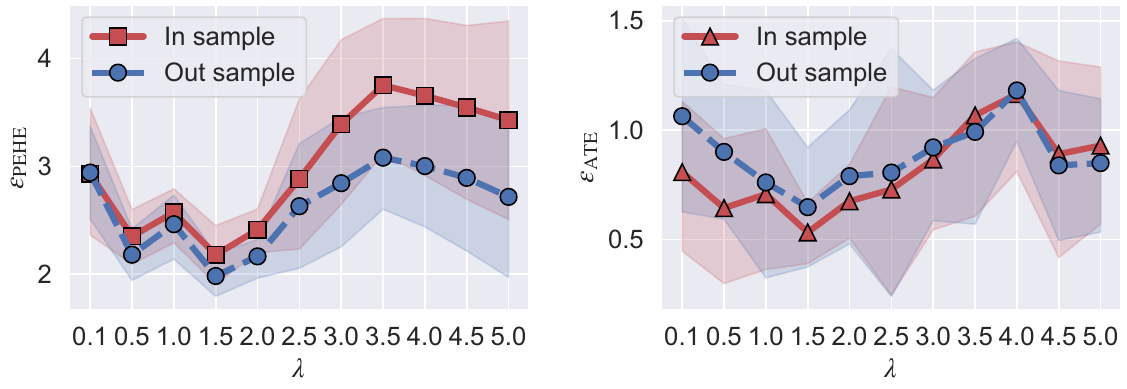}}
    \subfigure[Performance with varying $\kappa$]{\includegraphics[width=0.49\linewidth]{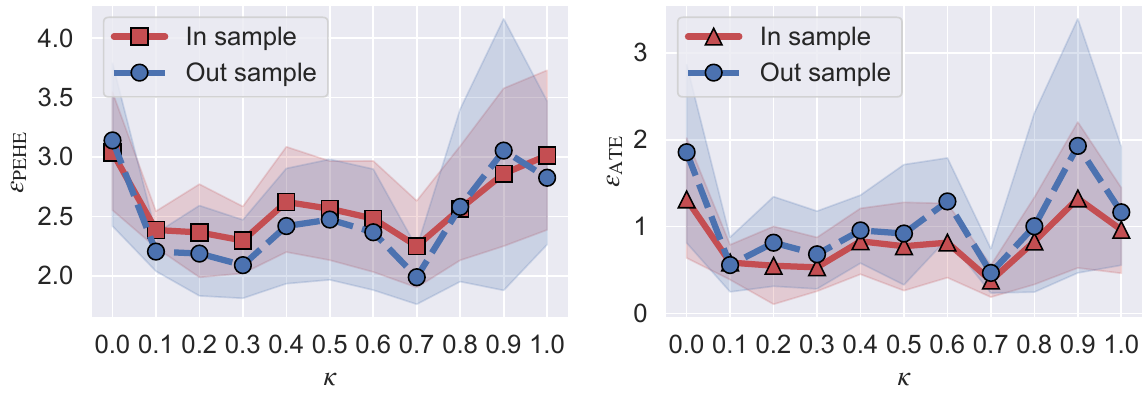}}
    \caption{Parameter sensitivity of the PPR module on the ACIC dataset, with focus on $\lambda$ and $\kappa$. The lines and shaded areas indicate the mean values and 90\% confidence intervals, respectively.}\label{fig:hparam1}
\end{figure*}
\begin{figure*}
    \centering
    \subfigure[Performance with $\kappa=0.3$.]{\includegraphics[width=0.49\linewidth]{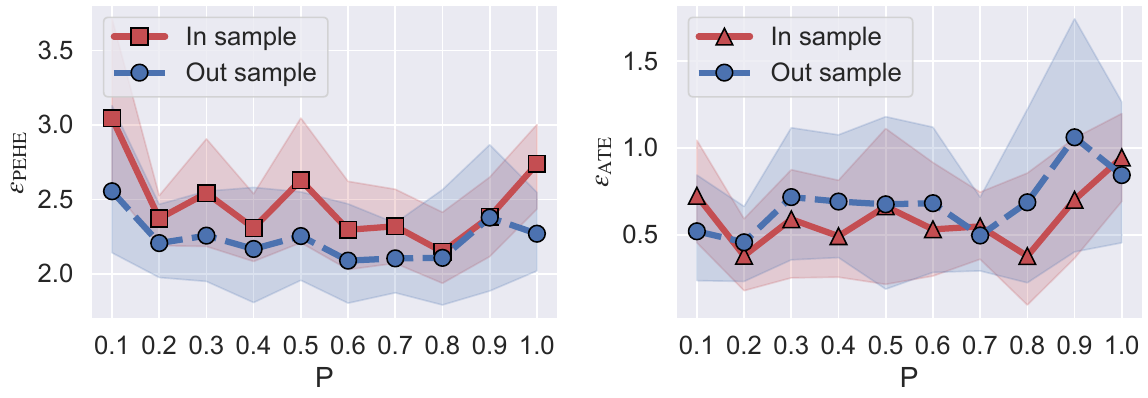}}
    \subfigure[Performance with $\kappa=0.7$.]{\includegraphics[width=0.49\linewidth]{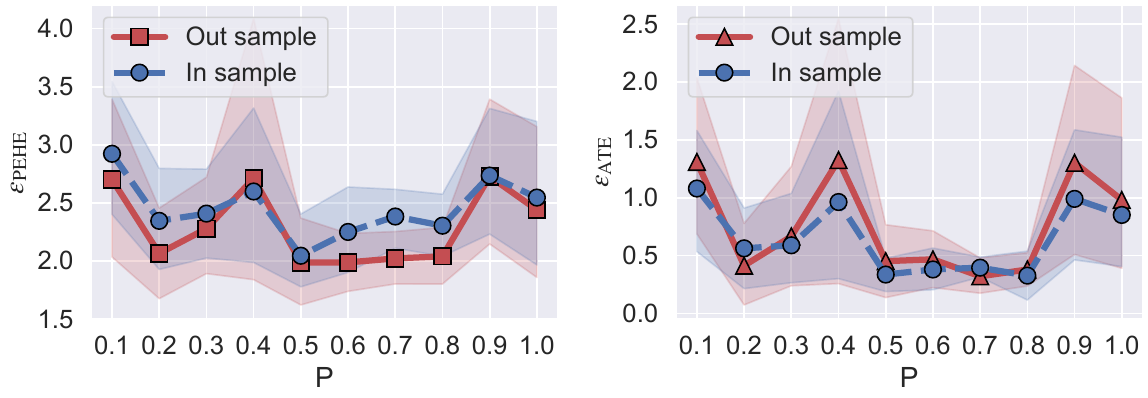}}
    \caption{Parameter sensitivity of the ISP module on the ACIC dataset, where $\mathrm{P}$ stands for the ratio of dimensionality reduction. The lines and shaded areas indicate the mean values and 90\% confidence intervals, respectively.}\label{fig:hparam2}
\end{figure*}

In Table~\ref{tab:ablation}, we examine the contributions of individual components of CFR-Pro  on the ACIC benchmark. Our study builds upon the CFR-Wass model~\cite{cfr}, a canonical approach that employs the Wasserstein discrepancy to align treatment groups within the representation space. 
\begin{itemize}[leftmargin=*]
    \item In CFR$^\dagger$, we enhance CFR by involving PPR, where the Wasserstein discrepancy $\mathbb{W}$ is replaced by the fused Gromov Wasserstein discrepancy in \eqref{eq:fgw}. A significant performance improvement is observed, where the out-of-sample $\epsilon_\mathrm{PEHE}$ decreases from 3.3187 to 2.5193, underscoring the utility of local proximity when constructing balanced representations.
    \item In CFR$^\ddagger$, we enhance CFR by incorporating ISP. Similarly, a huge performance improvement is observed. For instance, the out-of-sample $\epsilon_\mathrm{PEHE}$ decreases from 3.3187 to 2.5983, $\epsilon_\mathrm{ATE}$ decreases from 1.3581 to 0.4551. These performance improvements showcase the utility of ISP to mitigate the curse of dimensionality and generate reliable discrepancy estimates. 
    \item In CFR-Pro, we synthesize PPR and ISP into a unified framework. It maintains the advantages of each individual component and achieves the best overall performance compared to other variants. 
\end{itemize}

\subsection{In-depth analysis}\label{sec:hyper1}

\textbf{Analysis on PPR.} We analyze the performance of the PPR component focusing on its two main hyperparameters: $\lambda$ and $\kappa$, which respectively control the strength of distribution alignment and proximity preservation as detailed in \eqref{eq:fgw}. The results of a sensitivity study for these parameters are depicted in Figure~\ref{fig:hparam1}, and the key findings are summarized below.

\begin{itemize}[leftmargin=*]
    \item Increasing the value of $\lambda$ from 0.1 to 1.5 leads to a notable decrease in the out-of-sample $\epsilon_{\mathrm{PEHE}}$, from approximately 3.0 to 2.2. However, further increases in $\lambda$ result in a rise in estimation error. The phenomenon indicates that proper distribution alignment is effective to enhance the performance of HTE estimators. However,  overly emphasizing distribution balancing within a multi-task learning framework can degrade the accuracy of factual outcomes and, consequently, treatment effect estimates within a multi-task learning framework.
    \item A similar trend is observed for the strength of proximity preservation. Increasing $\kappa$ from 0 to 0.1 renders a huge performance improvement. The performance gain is consistent in a broad range from 0.1 to 0.7. However, further increasing $\kappa$ beyond this range leads to performance reduction. It results from an excessive emphasis on the local proximity in OT formulation, which reduces the focus on global discrepancy and hinders the reduction of selection bias.
\end{itemize}

\textbf{Analysis on ISP.}  The characteristics of ISP are governed by the hyperparameter \( \mathrm{P} \), which controls the extent of dimensionality reduction.
We conduct a hyperparameter study for \( \mathrm{P} \) in Figure~\ref{fig:hparam2}, and the key observations are summarized below.

\begin{itemize}[leftmargin=*]
    \item Dimensionality reduction effectively enhances model performance. Specifically, as \( \mathrm{P} \) is decreased from 1 to 0.7, there is a notable improvement in the estimation accuracy, with the out-of-sample \( \epsilon_{\mathrm{ATE}} \) diminishing from approximately 0.8 to about 0.5. This improvement is primarily due to effective handling of the curse of dimensionality, which in turn facilitates a more accurate estimation of discrepancies using minibatch samples. However, excessive reduction in dimensionality can lead to substantial information loss, and thereby suboptimal estimates. 
    \item Interestingly, there is a relationship between the weight of PPR (\( \kappa \)) and the optimal setting for \( \mathrm{P} \). As \( \kappa \) increases, which shifts the discrepancy measure \( \mathbb{P} \) closer to the Gromov-Wasserstein discrepancy, the curse of dimensionality becomes more pronounced: the Gromov term relies heavily on unit-wise distances to compute local proximity, making dimensionality reduction increasingly crucial. Consequently, the optimal \( \mathrm{P} \) value tends to decrease with larger \( \kappa \), underscoring the need to balance dimensionality reduction against the risk of significant information loss.
\end{itemize}

\section{Related works}\label{sec:relate}
\subsection{Overview of HTE estimation}
The central challenge of HTE estimation is mitigating treatment selection bias by balancing treated and untreated groups. Current solutions include three main categories: reweighting-based, matching-based, and representation-based methods~\cite{li2024debiased}.

Reweighting-based methods primarily utilize propensity scores to achieve global balance between groups, which entails the estimation of propensity scores and the construction of unbiased estimators~\cite{wustable,lirelaxing}. Propensity scores are typically estimated using logistic regression~\cite{mtlips,dual,dai2022generalized,chen2021autodebias}, with improvements via feature selection~\cite{shortreed2017outcome,wang2023out,wang2023select}, joint optimization~\cite{mtlips,zhanguser,liremoving}, and alternative training techniques~\cite{zhu2020unbiased}. The inverse propensity score method exemplifies the construction of unbiased estimator~\cite{ips}. However, it suffers from high variance at low propensity scores and bias with incorrect estimates~\cite{wangescm}. To handle these issues, doubly robust estimators and variance reduction techniques have been developed~\cite{dr,mrdr}. 

Matching-based methods construct locally balanced distributions by matching comparable units from different groups. These methods mainly differ in terms of the incorporated similarity measures. Propensity score matching for instance uses estimated propensity scores for calculating unit (dis)similarity~\cite{psm}. Tree-based methods, such as CausalForest~\cite{wager2018estimation}, can also viewed as a matching-based method employing adaptive similarity measures. Despite their effectiveness, the computational intensity of these methods restricts their scalability in large-scale applications~\cite{nnm,rnnm,wustable}.

Representation-based methods seeks for a mapping to a latent space where distributional discrepancies are minimized. Initial approaches advocate maximum mean discrepancy and vanilla Wasserstein discrepancy~\cite{bnn,cfr}. Enhancements have been made by integrating feature selection~\cite{disentangle,mim}, representation decomposition~\cite{wuite,disentangle}, and adversarial training~\cite{ganite}. However, local proximity is a critical yet scarcely investigated aspect to facilitate learning representation. A notable pioneer is SITE~\cite{site}, which employs the PDDM metric~\cite{pddm} to depict local proximity.



\subsection{Optimal transport for HTE estimation}

Recent advances in optimal transport (OT) have significantly impacted causality studies, leading to the development of innovative HTE estimators~\cite{ot4discovery}. One research direction involves using OT to enhance reweighting~\cite{otite1,yan2024exploiting,yang2024revisiting} and matching~\cite{otite4} strategies. A related work proposed by \citet{yan2024exploiting} uses Gromov discrepancy to adjust the transport matrix for HTE estimation. However, they focuses on reweighting samples within the covariate space, which contrasts with our approach that focuses on aligning representations in the latent space. Additionally, the issue of curse of dimensionality remains a limitation of \cite{yan2024exploiting} but handled in our work. These factors differentiate our work with \citet{yan2024exploiting}.

Another line of works, which are typically related to us, advocates building balanced representations with OT~\cite{otite3,causalot,escfr}. \citet{causalot} for instance use OT to align factual and counterfactual distributions; \citet{escfr} apply OT to achieve HTE estimation while minigating noise and unobserved confounding effects. Despite this progress, these approaches generally adhere to the traditional Kantorovich problem, similar to~\cite{cfr}, focusing on global alignment while often neglecting both local proximity and the curse of dimensionality. Therefore, developing OT formulations for HTE estimation remains a fruitful avenue for future research.

\section{Conclusion}
Representation learning has emerged as a pivotal approach for HTE estimation. However, current methods often overlook crucial aspects such as local proximity and the curse of dimensionality, which are essential for adequately addressing treatment selection bias. To bridge this gap, a principled approach known as CFR-Pro, based on a generalized OT problem, has been developed. Extensive experiments validate that CFR-Pro handles both problems effectively and outperforms prevalent baseline models.

There are two promising research avenues for further investigation. The first involves the integration of normalizing flows for representation mapping, since their invertibility effectively aligns with the foundational assumptions of counterfactual regression \cite{cfr}. The second avenue focuses on the practical application of our methodology to industrial contexts, specifically for bias mitigation in recommendation systems~\citep{escm}. 

\section*{Acknowledgements}
This work is supported
by National Key R\&D Program of China (2022ZD0160300) and the NSF China (No.
62276004, 623B2002).

\bibliographystyle{ACM-Reference-Format}
\balance
\bibliography{ref,causal}

\newpage
\setcounter{page}{1}
\appendix
\onecolumn
\section{Theoretical Justification} \label{apx_A}

\subsection{Notations and preliminaries on HTE estimation}\label{apdx_assump}
Here, we formalize the definitions, assumptions, and pertinent lemmas in the domain of HTE estimation from observational data. Building on the notations introduced in Section~\ref{sec:causal}, consider an individual with covariates $x$ exhibiting two potential outcomes: $Y_1(x)$ if treated and $Y_0(x)$ otherwise; CATE is defined as the difference between these outcomes, interpreted as $\tau(x) := \E\left[Y_1 - Y_0 \mid x \right]$.

\begin{definition}\label{def:mapapp}
    Let $\psi:\mathcal{S}_X \rightarrow \mathcal{S}_R$ denote a representation mapping that transforms covariates $X$ into a representation space $R=\psi(X)$. Define $\phi_T:\mathcal{R} \times \mathcal{T} \rightarrow \mathcal{Y}$ as an outcome mapping that correlates these representations and treatment states to their respective factual outcomes, with $Y_1 = \phi_1(R)$ and $Y_0 = \phi_0(R)$.
\end{definition}
\begin{definition}\label{def:perunitlossA}
The expected loss for the units with covariates $x$ and treatment indicator $t$ is: $\loss(x,t):= \int (Y_t-\phi_t(\psi(x)))^2\cdot\rho(Y_t\mid x)\,dY_t$. Then,
the expected factual outcome estimation error for treated, untreated and all units are:
\begin{equation}
    \label{eq:error1}
    \begin{aligned}
        \epsilon^\mathrm{T=1}_\mathrm{F}(\psi,\phi)  := \int{\loss(x,1)\cdot\pt(x)\, dx}, \quad
        \epsilon^\mathrm{T=0}_\mathrm{F}(\psi,\phi)  &:= \int  \loss(x,0)\cdot\pc(x)\, dx, \quad
        \epsilon_\mathrm{F}(\psi,\phi)  := \int  \loss(x,t)\cdot\rho(x,t)\, dxdt .
    \end{aligned}
\end{equation}
\end{definition}

An exemplar approach to CATE estimation is TARNet~\citep{cfr}. Based on the elements in \ref{def:mapapp}, it involves a representation mapping $\psi$ that is shared across the treated and untreated units, and outcome mappings $\phi_1$ and $\phi_0$ for treated and untreated units, respectively. It estimates the CATE as $\hat{\tau}_{\psi,\phi}(X) := \hat{Y}_1 - \hat{Y}_0,$ where $\hat{Y}_1 = \phi_1(\psi(X)), \hat{Y}_0 = \phi_0(\psi(X))$. The quality of CATE estimation is evaluated via the PEHE metric
\begin{equation}\label{eq:pehe_app}
    \epehe(\psi,\phi) := \int \left(\hat{\tau}_{\psi,\phi}(x) - \tau(x)\right)^2 {\color{black}\rho(x)} \, dx.
\end{equation}

During training, the factual error $\epsilon_\mathrm{F}(\phi,\psi)$ in Definition~\ref{def:perunitlossA} is optimized. However, treatment selection bias results in covariate distribution differences between treated and untreated groups, impeding model generalization across these groups. For example, an estimator $\phi_1$  trained solely on treated units may yield biased estimates of $\hat{\tau}$ when applied to untreated units, as shown in Figure~\ref{fig:problem}. To mitigate this bias, representation-based methods \cite{bnn,cfr} advocate for minimizing distribution discrepancies in the representation space and construct generalization bounds (see Theorem~\ref{thm:indtausqloss}). However, the IPM term in Theorem~\ref{thm:indtausqloss} is intractable for complex distributions. A common approach is to re-express IPM as the Wasserstein distance, as detailed in Lemma 1.

\begin{definition}\label{def:distri}
Let $\pt(x)$ and $\pc(x)$ denote the covariate distributions for the treated and untreated groups, respectively. Define $\pt_\psi(r)$ and $\pc_\psi(r)$ as the distributions of the representations $r = \psi(x)$, where $\psi$ is the representation mapping detailed in Definition~\ref{def:map}.
\end{definition}

\begin{definition}\label{def:ipm}
Given two distribution functions $\pt(x)$ and $\pc(x)$ supported over $\mathcal{X}$, and a sufficiently large function family $\mathcal{F}$, the Integral Probability Metric (IPM) induced by $\mathcal{F}$ is defined as: $\mathrm{IPM}_\mathcal{F}\left(\pt,\pc\right) = \sup_{f \in \mathcal{F}} \left|\int f(x) \left(\pt(x) - \pc(x)\right) \, dx \right|.$
\end{definition}

\begin{thm}\label{thm:indtausqloss}
Suppose $\mathcal{F}$ is a function family sufficiently large to include $\frac{1}{B_\psi} \cdot \loss(x,t)$ for $t \in \{0,1\}$, \citet{cfr} demonstrate that:
\begin{equation}
    \epehe(\psi, \phi) \leq 2\left(\epsilon_\mathrm{F}^\mathrm{T=0}(\psi, \phi) + \epsilon_\mathrm{F}^\mathrm{T=1}(\psi, \phi) + B_\psi  \mathrm{IPM}_\mathcal{F} \left( \pt_\psi, \pc_\psi \right) - 2\sigma^2_Y\right),
\end{equation}
where $\epsilon_\mathrm{F}^\mathrm{T=0}$ and $\epsilon_\mathrm{F}^\mathrm{T=1}$ are defined according to Definition~\ref{def:perunitlossA}, and $\pt_\psi(r)$ and $\pc_\psi(x)$ are specified in Definition~\ref{def:distri}.
\end{thm}

\begin{lem}\label{lem:ipm}
    Given two distribution functions $\p_1(x)$ and $\p_2(x)$ supported over $\mathcal{X}$, and letting $\mathcal{F}$ be the family of $1$-Lipschitz functions, we have $\mathrm{IPM}_\mathcal{F}\left(\p_1, \p_2\right) = \mathbb{W}\left(\p_1, \p_2\right)$, i.e., the IPM induced by $\mathcal{F}$ is equivalent to the Wasserstein distance $\mathbb{W}$ \cite{villani2009optimal}.
\end{lem}

\subsection{Theoretical results}\label{apdx_thm}

Theorem~\ref{thm:indtausqloss} assumes access to the entire populations of treated and untreated groups to calculate the distribution discrepancy. However, in training neural networks, parameters are typically updated using stochastic gradient methods on mini-batches rather than the full dataset. This raises concerns about the validity of Theorem~\ref{thm:indtausqloss} when applied at the mini-batch level.
Recent studies have investigated the sample complexity of various discrepancy measures, such as the Wasserstein distance (see Lemma~\ref{lem:wcomp}) and Gromov discrepancy (see Lemma~\ref{lem:gwcomp}). Building on these insights, we propose Theorem~\ref{thm:boundapp}, which extends Theorem~\ref{thm:indtausqloss} to the specific Fused Gromov-Wasserstein (FGW) discrepancy used in this work. This theorem examines the sample complexity of FGW when only a small mini-batch sample is available.

\begin{lem}\label{lem:wcomp}
    Consider two measures $\alpha$ and $\beta$ with compact supports $\mathcal{S}_\alpha \in\mathbb{R}^\mathrm{d}$ and $\mathcal{S}_\beta \in\mathbb{R}^\mathrm{d}$. Let $C=\mathrm{diam}(\mathcal{S}_\alpha)\vee \mathrm{diam}(\mathcal{S}_\beta)$, considering the case where  $\mathrm{d}>4$, we have:
    \begin{equation}
        \mathbb{E}\left[\left|\mathbb{W}(\alpha, \beta)^2-\mathbb{W}\left(\hat{\alpha}_n, \hat{\beta}_n\right)^2\right|\right] \lesssim n^{-\frac{2}{\mathrm{d}}},
    \end{equation}
    where the notation $\lesssim$ hides constants that is independent to the number of samples $n$. $\alpha_n$ and $\beta_n$ are empirical distributions of $\alpha$ and $\beta$ with $n$ i.i.d. samples.
\end{lem}
\begin{lem}\label{lem:gwcomp}
    Consider two measures $\alpha$ and $\beta$ with compact supports $\mathcal{S}_\alpha \in\mathbb{R}^\mathrm{d}$ and $\mathcal{S}_\beta \in\mathbb{R}^\mathrm{d}$. Let $C=\mathrm{diam}(\mathcal{S}_\alpha)\vee \mathrm{diam}(\mathcal{S}_\beta)$, considering the case where  $\mathrm{d}>4$, we have:
    \begin{equation}
        \mathbb{E}\left[\left|\mathbb{G}(\alpha, \beta)^2-\mathbb{G}\left(\hat{\alpha}_n, \hat{\beta}_n\right)^2\right|\right] \lesssim \frac{C^4}{\sqrt{n}}+\left(1+C^4\right) n^{-\frac{2}{\mathrm{d}}},
    \end{equation}
    where the notation $\lesssim$ hides constants that is independent to the number of samples $n$. $\alpha_n$ and $\beta_n$ are empirical distributions of $\alpha$ and $\beta$ with $n$ i.i.d. samples.
\end{lem}

\begin{thm}\label{thm:boundapp}
    Let $\psi$ and $\phi$ be the representation mapping and factual outcome mapping, respectively;
    $\hat{\mathbb{W}}_\psi$ be the group discrepancy at a mini-batch level.
    With the probability of at least $1-\delta$, we have:
    \begin{equation}\label{eq:peheBoundapp}
        \epehe(\psi,\phi)
        \leq 2[
            \epsilon^{T=1}_\mathrm{F}(\psi,\phi) + \epsilon^{T=0}_\mathrm{F}(\psi,\phi) + B_{\psi, \kappa} \mathbb{F}(\mathcal{R}_\psi^{T=0}, \mathcal{R}_\psi^{T=1})
            -2\sigma^2_{Y}+\mathcal{O}({N}^\mathrm{-\frac{2}{\mathrm{d}}})
            ],
    \end{equation}
    where $\epsilon^{T=1}_\mathrm{F}$ and $\epsilon^{T=0}_\mathrm{F}$ are the expected errors of factual outcome estimation,
    $N$ is the batch size, $\sigma^2_Y$ is the variance of outcomes, 
    $B_{\psi, \kappa}$ is a constant term, and $\mathcal{O}(\cdot)$ is a sampling complexity term.
\end{thm}

\begin{proof}
    According to Theorem~\ref{thm:indtausqloss} we have:
    \begin{equation}
        \epehe(\psi,\phi) \leq 2\left(\epsilon_\mathrm{F}^\mathrm{T=0}(\psi,\phi) +\epsilon_\mathrm{F}^\mathrm{T=1}(\psi,\phi)+  B_\psi  \mathrm{IPM}_\mathcal{F} \left( \pt_\psi, \pc_\psi \right) - 2\sigma^2_Y\right) .
    \end{equation}

    Assuming that there exists a constant $B_\psi>0$, such that for $t \in \{0,1\}$, $\frac{1}{B_\psi} \cdot \loss(x,t)$ belongs to the family of 1-Lipschitz functions.
    According to Lemma~\ref{lem:ipm}, we have
    \begin{equation}\label{eq:peheBoundapp2}
        \epehe(\psi,\phi) \leq 2\left(\epsilon_\mathrm{F}^\mathrm{T=0}(\psi,\phi) +\epsilon_\mathrm{F}^\mathrm{T=1}(\psi,\phi)+  B_\psi  \mathbb{W}\left( \pt_\psi, \pc_\psi \right) - 2\sigma^2_Y\right) .
    \end{equation}

    Following Definition~\ref{def:empirical}, let $\mathcal{R}_\psi^{T=1}$ and $\mathcal{R}_\psi^{T=0}$ be the empirical distributions of representations at a mini-batch level, both containing $n$ units. 
    Then, according to Lemma~\ref{lem:wcomp} and \ref{lem:gwcomp}, we have:
    \begin{equation}\label{eq:empirical1}
        \begin{aligned}
            \mathbb{W}\left( \pt_\psi, \pc_\psi \right)
            & \leq \frac{1}{\kappa} \left(\kappa*\mathbb{W}\left( \pt_\psi, \pc_\psi \right) + (1-\kappa) * \mathbb{G}\left( \pt_\psi, \pc_\psi \right) \right)\\
            & \leq \frac{1}{\kappa} \left(\kappa*\mathbb{W}\left( \mathcal{R}_\psi^{T=1}, \mathcal{R}_\psi^{T=0} \right) + (1-\kappa) * \mathbb{G}\left( \mathcal{R}_\psi^{T=1}, \mathcal{R}_\psi^{T=0} \right) + \frac{C^4}{\sqrt{n}}+\left(2+C^4\right) n^{-\frac{2}{\mathrm{d}}} \right)\\
            & \leq \frac{1}{\kappa} \left(\mathbb{F}\left( \mathcal{R}_\psi^{T=1}, \mathcal{R}_\psi^{T=0} \right) + \frac{C^4}{\sqrt{n}}+\left(2+C^4\right) n^{-\frac{2}{\mathrm{d}}}\right),
        \end{aligned}
    \end{equation}
    where $\mathbb{W}$ denotes the Wasserstein discrepancy, $\mathbb{G}$ denotes the Gromov discrepancy, $\mathbb{F}$ denotes the fused Wasserstein discrepancy. These discrepancies will be introduced in the next section.
    Notably, there are two terms in the cost function of $\mathbb{F}$ as per \eqref{eq:fgw}, which corresponding to  the cost functions of $\mathbb{W}$ and $\mathbb{G}$, respectively. Therefore, $\mathbb{W}$ and $\mathbb{G}$ can be viewed as minimizing the two terms of the cost function of $\mathbb{F}$ individually, which often yields smaller values.

    Denote $B_{\psi,\kappa}=B_{\psi}/\kappa$. Combing \eqref{eq:empirical1} and~\eqref{eq:peheBoundapp2}, we have
    \begin{equation}\label{eq:empirical3}
        \begin{aligned}
            \epehe(\psi,\phi)
        \leq 2[
            \epsilon^{T=1}_\mathrm{F}(\psi,\phi) + \epsilon^{T=0}_\mathrm{F}(\psi,\phi) + B_{\psi, \kappa} \mathbb{F}(\mathcal{R}_\psi^{T=0}, \mathcal{R}_\psi^{T=1})
            -2\sigma^2_{Y}+\mathcal{O}(n^\mathrm{-\frac{2}{\mathrm{d}}})
            ],
        \end{aligned}
    \end{equation}
    where $\mathcal{O}(n^\mathrm{-\frac{2}{\mathrm{d}}})=\frac{1}{\kappa} \left(\frac{C^4}{\sqrt{n}}+\left(2+C^4\right) n^{-\frac{2}{\mathrm{d}}}\right)$.    
    The proof is completed.
\end{proof}

\end{document}